\documentclass[a4paper,10pt]{article}

\usepackage{mathrsfs, amssymb}
\usepackage{amsthm}
\usepackage[intlimits]{empheq}
\usepackage{bbm} 
\usepackage{amsmath}
\usepackage{xcolor}
\usepackage[scientific-notation=true]{siunitx}
\usepackage{float}
\usepackage{dsfont}	
\usepackage[numbers]{natbib}
\usepackage[nottoc,numbib]{tocbibind}
\usepackage{placeins}
\usepackage{subcaption}
\usepackage[inner=2.0cm,outer=2.0cm,top=3cm,bottom=3cm]{geometry}
\usepackage{url}
\usepackage{hyperref}
\usepackage{footmisc}
\usepackage{comment}

\newenvironment{acknowledgement}{\begin{small}\paragraph{Acknowledgements}}%
                           {\end{small}}

\newtheorem{theorem}{Theorem}[section]

\newtheorem{prop}{Proposition}

\newtheorem{lemma}{Lemma}

\let\svthefootnote\thefootnote

\allowdisplaybreaks

\begin{document}

\title{Understanding recent deep-learning techniques for identifying collective variables of molecular dynamics}
\date{}
\author{
  Wei Zhang\,$^{*, \ddagger}$
  \and
Christof Sch\"utte\,$^{*,\ddagger}$
}
\maketitle
\let\thefootnote\relax
\footnotetext{$^*$Zuse Institute Berlin, Takustrasse 7, 14195 Berlin, Germany}
\footnotetext{$^\ddagger$Institute of Mathematics, Freie Universit\"{a}t Berlin, Arnimallee 6, 14195 Berlin, Germany}
\footnotetext{\, Email: wei.zhang@fu-berlin.de, christof.schuette@fu-berlin.de}

\let\thefootnote\svthefootnote
\renewcommand{\thefootnote}{\arabic{footnote}}

\abstract{
High-dimensional metastable molecular system can often be characterised by a few features of the system, i.e.\ collective variables (CVs).
Thanks to the rapid advance in the area of machine learning and deep learning, various deep learning-based CV identification techniques have been developed in recent years, allowing accurate modelling and efficient simulation of complex molecular systems. In this paper, we look at two different categories of deep learning-based approaches for finding CVs, either by computing leading eigenfunctions of infinitesimal generator or transfer operator associated to the underlying dynamics, or by learning an autoencoder via minimisation of reconstruction error. We present a concise
overview of the mathematics behind these two approaches and conduct a comparative numerical study of these two approaches on illustrative examples.
}


\section{Introduction}
\label{sec-intro}

Molecular dynamics (MD) simulation is a mature computational technique for the
study of biomolecular systems. It has proven valuable in a wide range of
applications, e.g.\ understanding functional mechanisms of proteins and discovering new drugs~\cite{Durrant2011-md-drug,HOLLINGSWORTH20181129-md-for-all}. However,
the capability of direct (all-atom) MD simulations is often limited, due to
the disparity between the tiny step-sizes that the simulations have to adopt in
order to ensure numerical stability and the large timescales on which the
functionally relevant conformational changes of biomolecules, such as protein folding, typically occur. 

One general approach to overcome the aforementioned challenge in MD
simulations is by utilizing the fact that in many cases the dynamics of a
high-dimensional metastable molecular system can be characterised by a few
features, i.e.\ collective variables (CVs) of the system. In deed, many
enhanced sampling methods (see~\cite{enhanced-sampling-for-md-review} for a
review) and approaches for building surrogate
models~\cite{perspective-noid-cg,effective_dynamics,effective_dyn_2017,non-markovian-modeling-pfolding-netz}
rely on knowing a set of CVs of the underlying molecular system. 
While empirical approaches and physical/chemical intuition are still widely
adopted in choosing CVs (e.g.\ mass centers, bonds, or angles), it is often
difficult or even impossible to intuit biomolecular systems in real-life
applications due to their high dimensionality, as well as structural and dynamical complexities. 

Thanks to the availability of numerous molecular data being generated and the
rapid advance of machine learning techniques, data-driven automatic
identification of CVs has attracted considerable research interests. Numerous machine
learning-based techniques for CV identification have emerged, such as the well-known principal component analysis (PCA)~\cite{pca-review}, diffusion maps~\cite{COIFMAN20065},
ISOMAP~\cite{isomap}, sketch-map~\cite{sketch-map}, time-lagged independent
component analysis (TICA)~\cite{tica},  as well the
kernel-PCA~\cite{kernel-pca} and kernel-TICA~\cite{kernel-tica} using kernel
techniques. See \cite{NOE2017-cv-review,discovery-mountain-passes-clementi} for reviews. 
The recent developments mostly employ deep learning techniques and largely fall into two categories.
Methods in the first category are based on the operator approach for the study
of stochastic dynamical systems. These include VAMPnets~\cite{vampnet} and the
variant state-free reversible VAMPnets (SRV)~\cite{state-free-vampnets}, the
deep-TICA approach~\cite{deep-tica}, and ISOKANN~\cite{isokann}, which are
capable of learning eigenfunctions of Koopman/transfer operators. The authors
of this paper have also developed a deep learning-based method for learning
eigenfunctions of infinitesimal generator associated to overdamped Langevin dynamics~\cite{eigenpde-by-nn-ZHANG}.
Methods in the second category combine deep learning with dimension reduction techniques, typically by training autoencoders~\cite{ae-kramer}. For instance, several approaches are proposed to
iteratively train autoencoders and improve training data by ``on-the-fly'' enhanced sampling.
These include the Molecular Enhanced Sampling with Autoencoders (MESA)~\cite{enhanced-sampling-with-ae}, Free Energy Biasing and
Iterative Learning with Autoencoders (FEBILAE)~\cite{chasing-cv-using-ae-and-biased-traj}, the method based on the predictive information bottleneck
framework~\cite{Wang2019-past-future-info-bottleneck}, the Spectral Gap Optimisation of Order Parameters (SGOOP)~\cite{spectral-gap-order-parameters-for-sampling-tiwary}, the deep Linear Discriminant Analysis (deep-LDA)~\cite{data-cv-for-enhanced-bonati2020}.
Besides, various generalized autoencoders are proposed, such as the extended autoencoder (EAE) model~\cite{extended-ae-bolhuis}, the time-lagged (variational)
autoencoder~\cite{time-lagged-ae-noe,var-encoding-pande}, Gaussian mixture variational autoencoder~\cite{Bozkurt-Varolg-ne2020}, and EncoderMap~\cite{encodermap}. 

Motivated by these rapid advances, in this paper we study the aforementioned
two categories of deep learning-based approaches for finding CVs, i.e.\
approaches for computing leading eigenfunctions of infinitesimal generator or
transfer operator associated to the underlying dynamics and approaches that
learn an autoencoder via minimisation of reconstruction error. We focus on
theoretical aspects of these approaches in order to gain better understanding
on their capabilities and limitations.  

The remainder of this article is organized as follows. In Section~\ref{sec-eigen}, we present an overview of the approaches for CV identification based on computing eigenfunctions. 
We give a brief introduction to infinitesimal generator and transfer operator,
then we discuss motivations for the use of eigenfunctions as CVs in studying
molecular kinetics, and finally we present variational characterisations as well as loss functions for learning eigenfunctions.
In Section~\ref{sec-ae}, we study autoencoders. We discuss the connection with PCA and present a characterisation of the optimal (time-lagged) autoencoder.
In Section~\ref{sec-examples}, we illustrate the numerical approaches for learning eigenfunctions and autoencoders by applying them to two simple yet illustrative systems.
Appendix~\ref{appsec-proofs} contains the proofs of two lemmas in Section~\ref{sec-eigen}.

\section{Eigenfunctions as CVs for the study of molecular kinetics on large timescales}
\label{sec-eigen}

In this section, we consider eigenfunctions of infinitesimal generator and transfer operator that are associated to the underlying dynamics.
We begin by introducing relevant operators, whose eigenfunctions will be
the focus of this section. After that we present two different perspectives,
which motivate the use of eigenfunctions as CVs to study molecular kinetics on large timescales. Finally, we discuss variational
formulations of leading eigenvalues and eigenfunctions, which will be useful in designing loss functions for training artificial neural networks.

\subsection{Operator approach}
\label{subsec-operators}

\paragraph{Generator} Molecular dynamics can be modelled by stochastic differential equations (SDEs). For both simplicity and mathematical convenience, 
we consider here the following SDE, often called the overdamped Langevin dynamics, 
\begin{equation}
  dX_s = -\nabla V(X_s) ds + \sqrt{2\beta^{-1}} dW_s\,,
  \label{sde-x}
\end{equation}
where $X_s \in \mathbb{R}^d$ is the system's state at time $s\in [0,+\infty)$,
$V: \mathbb{R}^d\rightarrow \mathbb{R}$ is a smooth potential function, $W_s$ is a $d$-dimensional Brownian motion that mimics the effect of noisy environment, and the noise strength
is determined by the parameter $\beta=(k_BT)^{-1}$ that is proportional to the inverse of the system's temperature $T$.
  We assume that dynamics \eqref{sde-x} is ergodic with respect to its unique
  invariant measure 
  \begin{align}
    d\mu(x) = \pi(x)dx, \quad \mathrm{with}~ \pi(x)= \frac{1}{Z} e^{-\beta V(x)}, \quad x\in \mathbb{R}^d\,,   \label{mu-invariant}
  \end{align}
  where $Z$ is a normalising constant.

 The infinitesimal generator of \eqref{sde-x} is a second-order differential operator, defined by
\begin{equation}
  \mathcal{L} f = -\nabla V \cdot \nabla f + \frac{1}{\beta}\Delta f = \frac{1}{\beta} \mathrm{e}^{\beta V} \mbox{div} (\mathrm{e}^{-\beta V} \nabla f)\,,
  \label{generator-l}
\end{equation}
for a test function $f: \mathbb{R}^d\rightarrow \mathbb{R}$. Dynamics \eqref{sde-x} is reversible, and its generator $\mathcal{L}$ is self-adjoint in $L^2(\mu)$ endowed with
the weighted inner product $\langle f, g\rangle_\mu:= \int_{\mathbb{R}^d} fg
d\mu$. In fact, using \eqref{mu-invariant}--\eqref{generator-l} and integration by parts, one can verify that
  \begin{equation}
       \langle (-\mathcal{L}) f, g\rangle_\mu = \langle f, (-\mathcal{L}) g\rangle_\mu = \frac{1}{\beta} \mathbf{E}_\mu \big(\nabla f\cdot \nabla g\big) \,,
       \label{l-selfadjoint}
  \end{equation}
for two $C^2$-smooth test functions $f, g: \mathbb{R}^d\rightarrow
\mathbb{R}$, where $\mathbf{E}_\mu(\cdot)$ denotes the mathematical expectation with respect to the measure $\mu$ in \eqref{mu-invariant}.
We also define the energy 
\begin{equation}
\mathcal{E}(f)=\frac{1}{\beta} \mathbf{E}_\mu \big(|\nabla f|^2\big)\,, \quad f: \mathbb{R}^d\rightarrow \mathbb{R}\,,
  \label{energy-generator}
\end{equation}
which is considered to be $+\infty$ when the right hand side in \eqref{energy-generator} is undefined. Under certain conditions on $V$, the operator $-\mathcal{L}$ has purely discrete spectrum, consisting of a sequence of eigenvalues~\cite{eigenpde-by-nn-ZHANG} 
\begin{equation}
0 = \lambda_0 < \lambda_1 \le \lambda_2 \le \cdots\,,
  \label{eigenvalues-l}
\end{equation}
with the corresponding (orthogonal and normalised) eigenfunctions
$\varphi_0\equiv 1, \varphi_1, \varphi_2, \cdots \in L^2(\mu)$. The leading nontrivial (nonzero) eigenvalues in \eqref{eigenvalues-l} determine the large timescales of the underlying dynamics, whereas the corresponding
eigenfunctions are closely related to metastable conformations.

\paragraph{Transfer operator}
In contrast to the discussion above based on SDEs, transfer operator approach offers an alternative way to study dynamical systems without specifying the
governing equations~\cite{transfer_operator,msm_generation} and is hence widely adopted in developing numerical algorithms. In this framework, one assumes that the trajectory data is sampled from an underlying (equilibrium) system whose
state $y$ at time $t+\tau$ given its state $x$ at time $t$ can be modelled as
a discrete-time Markovian process with transition density $p_\tau(y|x)$, for
all $t\ge 0$, where $\tau>0$ is called the lag-time and the process is assumed to be ergodic with respect to the
unique invariant distribution $\mu$ in \eqref{mu-invariant}. 
The transfer operator associated to this discrete-time Markovian process is
defined as~\cite{msm_generation}
\begin{equation}
  \mathcal{T} u (x) = \frac{1}{\pi(x)} \int_{\mathbb{R}^d} p_\tau(x|y) u(y) \pi(y)
  dy\,, \quad x \in \mathbb{R}^d
  \label{transfer-operator}
\end{equation}
for a density (with respect to $\mu$) $u: \mathbb{R}^d\rightarrow
\mathbb{R}^+$. We assume that the detailed balance condition is satisfied, i.e.\ $p_\tau(y|x)\pi(x) = p_\tau(x|y)\pi(y)$ for all
$x,y\in \mathbb{R}^d$. Then, we can derive
\begin{equation}
  \begin{aligned}
    \mathcal{T} u (x) =& \frac{1}{\pi(x)} \int_{\mathbb{R}^d} p_\tau(x|y) u(y) \pi(y) dy \\
    =& \int_{\mathbb{R}^d} p_\tau(y|x) u(y) dy \\
    =& \mathbf{E}(u(X_{\tau})|X_0 = x)\,, 
  \end{aligned}
  \label{transfer-operator-and-semigroup}
\end{equation}
which shows that in the reversible setting the transfer operator coincides with the semigroup operator (at time $\tau$) associated to the
underlying process~\cite{sz-entropy-2017}~\footnote{In the literature, the
expression in the last line of \eqref{transfer-operator-and-semigroup} is also
used to define Koopman operators for stochastic dynamics~\cite{Wu2020}. We stick
to the notion of transfer operator and note that both operators are identical for reversible
processes. We refer to~\cite{var-koopman-model,Wu2020} and the references therein for the study of stochastic dynamics using Koopman operators.}.
Similar to the generator, one can show that $\mathcal{T}$ is self-adjoint in $L^2(\mu)$ with respect to $\langle\cdot,\cdot\rangle_\mu$ (see \eqref{l-selfadjoint}). 
Also, in analogy to \eqref{energy-generator}, for a function $f\in L^2(\mu)$ we define the energy 
  \begin{equation}
    \mathcal{E}_\tau(f) = \frac{1}{2}\int_{\mathbb{R}^d} \int_{\mathbb{R}^d} \big(f(y) - f(x))^2 p_\tau(y|x) \pi(x) dx dy\,.
    \label{energy-transfer}
    \end{equation}
The following lemma provides an alternative expression of \eqref{energy-transfer} involving the transfer operator $\mathcal{T}$.
\begin{lemma}
  Denote by $I: L^2(\mu)\rightarrow L^2(\mu)$ the identity map.  For all $f \in L^2(\mu)$, we have 
  \begin{equation}
      \mathcal{E}_\tau(f) = \int_{\mathbb{R}^d} \big[(I-\mathcal{T})f(x)\big]
      f(x) d\mu(x) = \langle (I-\mathcal{T})f, f\rangle_\mu\,.
    \end{equation}
    \label{lemma-energy-transfer}
\end{lemma}

The proof of Lemma~\ref{lemma-energy-transfer} is straightforward and we present it in Appendix~\ref{appsec-proofs}. 

Lemma~\ref{lemma-energy-transfer} and \eqref{energy-transfer} imply that all eigenvalues of
$\mathcal{T}$ are no larger than one. We assume that the spectrum of $\mathcal{T}$ consists of discrete eigenvalues  
\begin{align}
  1=\nu_0 > \nu_1 \ge \cdots 
  \label{eigen-transfer-sequence}
\end{align}
and the largest eigenvalue $\nu_0=1$ (corresponding to the trivial eigenfunction $\varphi_0\equiv 1$) is non-degenerate.
These eigenvalues and their corresponding eigenfunctions are of great interest in applications, since they encode information about the timescales and metastable conformations of the underlying dynamics, respectively~\cite{transfer_operator,msm_generation}. 
For the process defined by SDE~\eqref{sde-x}, in particular, the transfer
operator and the generator satisfy $\mathcal{T}=\mathrm{e}^{\tau
\mathcal{L}}$, which implies that the eigenvalues of $\mathcal{T}$ and $-\mathcal{L}$ are related by $\nu_i=\mathrm{e}^{-\tau \lambda_i}$ with identical eigenfunctions $\varphi_i$, for $i\ge 0$~\cite{sz-entropy-2017}.

\subsection{Motivations to use eigenfunctions as CVs}
\label{subsec-motivations}

There is a large amount of literature on the study of eigenfunctions of infinitesimal generator, transfer operator, or Koopman operator. For the
transfer operator $\mathcal{T}$, for instance, many of these studies are
motivated by the connection between the (pairwise orthogonal and normalised) eigenfunctions and the action of $\mathcal{T}$ on test functions $f
\in L^2(\mu)$, i.e.\ in the reversible case, 
\begin{equation}
  \mathcal{T}^n f(x) = 
  \mathbf{E}(f(X_{n\tau})|X_0=x) = 
  \mathbf{E}_\mu(f) + \sum_{i=1}^{+\infty} \langle f, \varphi_i\rangle_\mu
  \nu_i^n \varphi_i(x)\,,\quad x\in \mathbb{R}^d, ~ n = 1,2, \dots\,.
\end{equation}
Since $\nu_1, \nu_2, \dots$ are all smaller than $1$, for large integers $n$,
the function $\mathcal{T}^n f$ is mainly determined by the leading eigenvalues
of $\mathcal{T}$ in \eqref{eigen-transfer-sequence} and the corresponding eigenfunctions. 
Therefore, knowing the leading eigenvalues and eigenfunctions of $\mathcal{T}$ helps study the map $\mathcal{T}^n$ for large $n$, which in turn 
helps understand the behavior of the underlying dynamics at large time $T=n\tau$. For Koopman operator, the leading eigenfunctions define the optimal linear Koopman model for features (functions)~\cite{Wu2020}.

Here, we contribute to this discussion by providing two different perspectives
that directly connect eigenfunctions to the underlying dynamics and to the choices of CVs. 
We assume that the dynamics satisfies SDE~\eqref{sde-x} and we will work with its generator $\mathcal{L}$. Most of the results below can be extended to a more
general setting, e.g.\ overdamped Langevin dynamics with state-dependent diffusion coefficients.
It is also possible to obtain parallel results for the discrete-time Markovian process involving the transfer operator $\mathcal{T}$~\footnote{This is an ongoing work that will be published in future.}.

Let $\xi = (\xi_1, \xi_2, \dots, \xi_k)^\top: \mathbb{R}^d\rightarrow \mathbb{R}^k$ be a smooth CV map, where $1 < k\ll d$. Ito's formula gives 
\begin{equation}
  d\xi(X_s) = \mathcal{L}\xi(X_s)\,ds + \sqrt{2\beta^{-1}} \nabla \xi(X_s) dW_s\,,
  \label{ito-xi}
\end{equation}
where $\nabla \xi(x) \in \mathbb{R}^{k\times d}$ denotes the Jacobian matrix of $\xi$ at $x\in \mathbb{R}^d$.
Given the projection dimension $k$, we are interested in finding a good CV map
$\xi$ that is both non-trivial and non-degenerate. In other words, the components $\xi_1,\xi_2,\dots, \xi_k$ of $\xi$
should be both non-constant and linearly independent. 
These two requirements can be met by imposing the following conditions (no loss of generality)
\begin{equation}
  \mathbf{E}_\mu(\xi_i) = 0 \,, \quad  \langle \xi_i, \xi_j\rangle_\mu =
  \mathbf{E}_\mu(\xi_i\xi_j)=\delta_{ij}\,, \quad 1 \le i \le j \le k\,.
  \label{orthonormality-condition}
\end{equation}

\paragraph{Optimal CVs for the study of slow motions}

For the first perspective, we relate the dynamics of \eqref{sde-x} on large timescales to the slow motions in it. 
This view suggests that a good CV map $\xi$ that captures the behavior of \eqref{sde-x} on large timescales should meet the following criteria:
\begin{center}
  \textit{ $\xi(X_s)$ evolves much more slowly comparing to the dynamics $X_s$ itself. } \qquad (C1)
\label{c1}
\end{center}
Since $\xi(X_s)$ satisfies \eqref{ito-xi}, to meet criteria (C1) it
is therefore natural to require the magnitude of both terms on the right hand side of
\eqref{ito-xi} to be small (in the sense of averages with respect to the
invariant distribution $\mu$ in \eqref{mu-invariant}). This can be formulated as an optimisation problem
\begin{equation}
  \min_{\xi_1, \dots, \xi_k} \sum_{i=1}^k \omega_i \int_{\mathbb{R}^d} \Big(|\mathcal{L}\xi_i|^2(x)
  + |\nabla\xi_i|^2(x)\Big) d\mu(x),  \quad \mbox{subject to \eqref{orthonormality-condition}}\,,
  \label{task-general}
\end{equation}
where $\omega_1 \ge \omega_2 \ge \dots \ge \omega_k >0$ are weights assigned
to the $k$ equations in \eqref{ito-xi}. One can choose the weights to be identical,
but using pairwise distinct weights could help eliminate non-uniqueness of the optimiser of \eqref{task-general} due to permutations.
We make the following claim concerning the optimiser of \eqref{task-general}.
\begin{prop}
  Assume that $-\mathcal{L}$ has purely discrete spectrum consisting of the eigenvalues in \eqref{eigenvalues-l}.
  Then, the minimum of \eqref{task-general} is attained by the first $k$ (non-trivial) eigenfunctions of $-\mathcal{L}$, i.e.\ when $\xi_i=\varphi_i$ for $i=1,\dots, k$.
  \label{prop-1}
\end{prop}
\begin{proof}
  Using the identities in \eqref{l-selfadjoint}, one can reformulate the optimisation problem \eqref{task-general} as 
\begin{equation}
  \min_{\xi_1, \dots, \xi_k} \sum_{i=1}^k \omega_i \big\langle [(-\mathcal{L})^2 + \beta (-\mathcal{L})] \xi_i, \xi_i\big\rangle_\mu,  \quad \mbox{subject to \eqref{orthonormality-condition}}\,.
  \label{task-general-reformulated}
\end{equation}
The conclusion follows once we show that the minimum of \eqref{task-general-reformulated} is attained when $\xi_i=\varphi_i$ for $i=1,\dots, k$.
  This can be achieved straightforwardly by repeating the proof of Theorem~\ref{thm-variational-form} in Section~\ref{subsec-var} below (the proof is given in~\cite{eigenpde-by-nn-ZHANG}) for the operator $(-\mathcal{L})^2 + \beta (-\mathcal{L})$
  and using the fact that both $(-\mathcal{L})^2 + \beta (-\mathcal{L})$ and $-\mathcal{L}$ have the same set of eigenfunctions.
\end{proof}

It is not difficult to see that the eigenfunctions $\varphi_1, \dots, \varphi_k$ actually minimise both terms in the objective \eqref{task-general} simultaneously (subject to \eqref{orthonormality-condition}). 
The following identity provides an explicit expression for the first term in \eqref{task-general}, which involves the operator $(-\mathcal{L})^2$. 

\begin{lemma}
  For any smooth function $f:\mathbb{R}^d\rightarrow \mathbb{R}$ such that $\langle (-\mathcal{L})^2 f, f\rangle_\mu <+\infty$, we have 
  \begin{align}
    \int_{\mathbb{R}^d} |\mathcal{L}f|^2 d\mu = \langle (-\mathcal{L})^2 f, f\rangle_\mu = \frac{1}{\beta}
    \int_{\mathbb{R}^d} \Big[\mathrm{Hess}V(\nabla f, \nabla f) + \frac{1}{\beta}|\nabla^2 f|^2_F \Big] d\mu\,, 
    \label{nice-identity}
  \end{align}
  where $\mathrm{Hess}V(\cdot, \cdot)$ is the Hessian operator of the
  potential $V$ and $|\nabla^2 f|_F$ denotes the Frobenius norm of the matrix $\nabla^2 f$.
\label{lemma-nice-identity}
\end{lemma}

The proof of Lemma~\ref{lemma-nice-identity} is given in Appendix~\ref{appsec-proofs}.
The integrand of the rightmost integral in \eqref{nice-identity} consists of the Hessian of $V$ and a regularising term.
Loosely speaking, since the eigenfunctions minimise \eqref{nice-identity}
subject to \eqref{orthonormality-condition}, \eqref{nice-identity} reveals the
connection between the (global) eigenfunctions of $(-\mathcal{L})$ and the (local) eigenvectors of the Hessian of the potential $V$. 

A final remark on \eqref{task-general} is that it does not rely on the
specific form of the SDE. Therefore, in principle it can be used as a criteria
of good CVs in the case where the SDE has a more general form, e.g.\ a
non-reversible SDE or underdamped Langevin dynamics. In these general
settings, it is interesting to study whether \eqref{task-general} can be solved efficiently using approaches such as physics-informed neural networks (PINN)~\cite{pinn}.

\paragraph{Optimal CVs for building effective dynamics}

The second perspective is inspired by the study of effective dynamics of \eqref{sde-x} using conditional expectations~\cite{effective_dynamics}.
Specifically, note that the SDE \eqref{ito-xi} of $\xi(X_s)$ is non-closed, in the sense
that terms on its right hand side still depend on the full state $X_s \in \mathbb{R}^d$.
The authors in~\cite{effective_dynamics} proposed an effective dynamics as a Markovian approximation of \eqref{ito-xi}, which is described by the SDE 
\begin{align}
    dz(s) = \widetilde{b}(z(s)) \,ds + \sqrt{2\beta^{-1}} \widetilde{\sigma}(z(s)) \,d\widetilde{w}(s)\,,
    \label{eff-dynamics}
\end{align}
where $\widetilde{w}(s)$ is a $k$-dimensional Brownian motion, the coefficients $\widetilde{b}: \mathbb{R}^k \rightarrow \mathbb{R}^k$ and $\widetilde{\sigma}\in \mathbb{R}^{k\times k}$ are defined by 
\begin{align*}
  \widetilde{b}_l(z) = \mathbf{E}_{\mu_z}(\mathcal{L}\xi_l)\,,~ 1 \le l \le k, \quad 
    (\widetilde{\sigma}\widetilde{\sigma}^\top)(z) = \mathbf{E}_{\mu_z}\big(\nabla \xi \nabla \xi^\top \big)\,, \quad \mathrm{for}~ z \in \mathbb{R}^k\,,
\end{align*}
respectively. In the above, for $z \in \mathbb{R}^k$, $\mathbf{E}_{\mu_z}(\cdot)$ denotes the conditional expectation on the level set $\Sigma_{z} = \big\{x\in \mathbb{R}^d \big|\xi(x) = z\big\}$ with respect to the so-called conditional measure $\mu_z$ on $\Sigma_z$:
\begin{align}
  \begin{split}
    d\mu_z(x) =& \frac{1}{Q(z)} \frac{e^{-\beta V(x)}}{Z} \Big[\det(\nabla \xi\nabla \xi^\top)(x)\Big]^{-\frac{1}{2}}\, d\nu_z(x) \\
    =&  \frac{1}{Q(z)} \frac{e^{-\beta V(x)}}{Z} \delta(\xi(x)-z) dx\,,
  \end{split}
  \label{mu-z}
  \end{align}
  where the first equality follows from the co-area formula, $\delta(\cdot)$ denotes the Dirac delta function, $Q(z)$ is a normalising constant, and $\nu_z$ denotes the surface measure on $\Sigma_z$. We refer to
  \cite{effective_dynamics,effective_dyn_2017} for detailed discussions about the definition and properties of the effective dynamics~\eqref{eff-dynamics}.

  Note that the effective dynamics \eqref{eff-dynamics} can be defined with a
  general CV map $\xi$. A natural question is how to choose $\xi$ such that
  the resulting effective dynamics is a good approximation of the original
  dynamics~\cite{Duong_2018,LEGOLL2017-pathwise,upanshu-nonreversible-2018,pathwiseEff-2018}.
  One way to quantify the approximation quality of \eqref{eff-dynamics} is by
  comparing its timescales to the timescales of the original dynamics~\cite{effective_dyn_2017,sz-entropy-2017}.
  For the overdamped Langevin dynamics \eqref{sde-x}, in particular, the infinitesimal generator of its effective dynamics \eqref{eff-dynamics}, denoted by
  $\widetilde{\mathcal{L}}$, is again self-adjoint in an appropriate Hilbert space~\cite{effective_dyn_2017}. Assume that $-\widetilde{\mathcal{L}}$ has purely discrete spectrum, which consists of eigenvalues $0=\widetilde{\lambda}_0 < \widetilde{\lambda}_1 \le \widetilde{\lambda}_2
  \le \cdots$, and let $\widetilde{\varphi}_i: \mathbb{R}^k\rightarrow \mathbb{R}$ be the corresponding orthonormal eigenfunctions. 
  The following result estimates the approximation error of the effective dynamics in terms of eigenvalues.
  \begin{prop}[\cite{effective_dyn_2017}]
    Recall the energy $\mathcal{E}$ defined in \eqref{energy-generator}. For $i=1,2,\dots$, we have 
  \begin{align}
    \lambda_i \le \widetilde{\lambda}_i  \le \lambda_i + \mathcal{E}(\varphi_i - \widetilde{\varphi}_i \circ \xi)\,.
 \end{align}
  In particular, when $\xi(x) = \big(\varphi_1(x), \varphi_2(x), \cdots, \varphi_k(x)\big)^\top \in \mathbb{R}^k$, we have $\widetilde{\lambda}_i = \lambda_i$, for\, $1 \le i \le k$.
    \label{prop-eff-eigen}
\end{prop}
Proposition~\ref{prop-eff-eigen} implies that, for a general CV map $\xi$, the
eigenvalues associated to the effective dynamics are always larger or equal to the
corresponding true eigenvalues, and the approximation error depends on the closeness between the corresponding
eigenfunctions (measured by the energy $\mathcal{E}$).
Also, choosing eigenfunctions associated to the original dynamics as the CV map $\xi$ yields the optimal effective
dynamics \eqref{eff-dynamics}, in the sense that it preserves the corresponding eigenvalues (timescales). 

\subsection{From variational characterisations to loss functions}
\label{subsec-var}

In the following, we discuss variational characterisations of eigenfunctions for both generator and transfer operator.
These characterisations are useful in developing numerical algorithms~\cite{frank_feliks_mms2013,feliks_variational_tensor2016}, in
particular in designing loss functions in recent deep learning-based approaches~\cite{vampnet,state-free-vampnets,eigenpde-by-nn-ZHANG}.

For the generator $\mathcal{L}$, note that we have already given a variational characterisation of the leading eigenfunctions $\varphi_1, \dots, \varphi_k$ 
in \eqref{task-general} thanks to Proposition~\ref{prop-1}.
However, as mentioned in Section~\ref{subsec-motivations}, the leading
eigenfunctions actually minimise both terms in \eqref{task-general} simultaneously and a simpler characterisation is preferred for numerical purposes.
In this regard, we record the following characterisation obtained in~\cite{sz-entropy-2017,eigenpde-by-nn-ZHANG}.
  \begin{theorem}
Let $k\in \mathbb{N}$ and $\omega_1 \ge \dots \ge \omega_k >0$.  Define $\mathcal{H}^1 := \big\{f\in L^2(\mu)\,\big|\, \mathbf{E}_\mu(f)=0,\, \langle (-\mathcal{L} f, f\rangle_\mu < +\infty\big\}$.
 We have 
    {\small
  \begin{align}
    \sum_{i=1}^k \omega_i\lambda_i =\min_{f_1,\dots, f_k\in \mathcal{H}^1} \sum_{i=1}^k \omega_i \mathcal{E}(f_i)\,,
    \label{variational-all-first-k}
  \end{align}
    }
    where $\mathcal{E}$ denotes the energy $\eqref{energy-generator}$, and the minimisation is over all $f_1,f_2,\dots, f_k\in \mathcal{H}^1$ such that 
    {\small
  \begin{equation} 
   \langle f_i,f_j\rangle_\mu = \delta_{ij}\,, \quad \forall i,j\, \in \{1,\dots,k\}\,.
    \label{f-ij-constraints}
  \end{equation}
    }
 Moreover, the minimum in \eqref{variational-all-first-k} is achieved when $f_i=\varphi_i$ for $1 \le i \le k$.
  \label{thm-variational-form}
\end{theorem}

To apply Theorem~\ref{thm-variational-form} in designing learning algorithms, we use the right hand side of \eqref{variational-all-first-k} as objective and
add penalty term to it in order to incorporate the constraints \eqref{f-ij-constraints}. 
In the end, we obtain the loss function that can be used to learn eigenfunctions of the generator by training neural networks:
    \begin{equation}
      \mathrm{Loss}(f_1, f_2, \dots, f_k) =
      \frac{1}{\beta}\sum_{i=1}^k\omega_i \frac{\mathbf{E}^{\textrm{data}}(|\nabla
      f_i|^2)}{\mathrm{Var}^{\mathrm{data}} (f_i)}  + \alpha\sum_{1\le i_1 \le i_2 \le k} 
	\Big(\mathrm{Cov}^{\mathrm{data}} \big(f_{i_1},f_{i_2} \big) - \delta_{i_1i_2}\Big)^2 \,,
	\label{loss-for-generator}
\end{equation}
  where $\alpha$ is a penalty constant, and $\mathbf{E}^{\mathrm{data}}$,
  $\mathrm{Var}^{\mathrm{data}}$, $\mathrm{Cov}^{\mathrm{data}}$ denote empirical estimators of mean, variance, and co-variance with respect to the measure $\mu$, respectively.
For brevity, we omit further discussions on the loss \eqref{loss-for-generator}, and we refer to \cite[Seciton 3]{eigenpde-by-nn-ZHANG} for more details.

For the transfer operator $\mathcal{T}$, using the same proof of
Theorem~\ref{thm-variational-form} (see the proof of \cite[Theorem 1]{eigenpde-by-nn-ZHANG}) and Lemma~\ref{lemma-energy-transfer} we can prove the following variational characterisation.
  \begin{theorem}
Let $k\in \mathbb{N}$ and $\omega_1 \ge \dots \ge \omega_k >0$. 
    Assume that $\mathcal{T}$ has discrete spectrum consisting of the eigenvalues in \eqref{eigen-transfer-sequence} with the corresponding eigenfunctions $\varphi_i$, $i\ge 0$.
    Define $L^2_0(\mu) := \big\{f\in L^2(\mu)\,\big|\, \mathbf{E}_\mu(f)=0\}$. We have 
  \begin{align}
    \sum_{i=1}^k \omega_i(1-\nu_i)=\min_{f_1,\dots, f_k\in L^2_0(\mu)} \sum_{i=1}^k \omega_i \mathcal{E}_\tau(f_i)\,,
    \label{variational-transfer-all-first-k}
  \end{align}
    where $\mathcal{E}_\tau$ is the energy in \eqref{energy-transfer} associated to $\mathcal{T}$, and the minimisation is over all $f_1,f_2,\dots, f_k\in
    L^2_0(\mu)$ under the constraints \eqref{f-ij-constraints}.
 Moreover, the minimum in \eqref{variational-transfer-all-first-k} is achieved when $f_i=\varphi_i$ for $1 \le i \le k$.
  \label{thm-variational-form-transfer}
\end{theorem}

We note that similar variational characterisations for eigenfunctions of transfer operator and Koopman operator have been studied in~\cite{state-free-vampnets,Wu2020}.

As in the case of generator, Theorem~\ref{thm-variational-form-transfer}
motivates the following loss function for learning eigenfunctions of the
transfer operator $\mathcal{T}$~\footnote{For overdamped dynamics~\eqref{sde-x},
we have $\frac{1 - \nu_i}{\tau}=\frac{1-e^{-\tau
\lambda_i}}{\tau} \approx \lambda_i$ when $\tau$ is small, where $\lambda_i$
is the corresponding eigenvalue of the generator. Based on this relation, we
include the constant $\frac{1}{\tau}$ in the first term of
\eqref{loss-for-transfer-operator}. }:

    \begin{equation}
      \mathrm{Loss}_\tau(f_1, f_2, \dots, f_k) =
      \frac{1}{2\tau}\sum_{i=1}^k\omega_i
      \frac{\mathbf{E}^{\textrm{data}}_{x\sim \mu, y\sim p_\tau(\cdot|x)}|f_i(y) - f_i(x)|^2}{\mathrm{Var}^{\mathrm{data}} (f_i)}  + \alpha\sum_{1\le i_1 \le i_2 \le k} 
	\Big(\mathrm{Cov}^{\mathrm{data}} \big(f_{i_1},f_{i_2} \big) - \delta_{i_1i_2}\Big)^2 \,,
	\label{loss-for-transfer-operator}
\end{equation}
  where $\mathbf{E}^{\textrm{data}}_{x\sim \mu, y\sim p_\tau(\cdot|x)}$
  denotes the empirical mean with respect to the joint distribution
  $p_\tau(y|x)d\mu(x)dy$, which can be estimated using time-series data (similar to \eqref{ae-reconstruction-error-tau} in Section~\ref{sec-ae}).

Compared to VAMPnets~\cite{vampnet}, the loss
\eqref{loss-for-transfer-operator} imposes orthogonality constraints
\eqref{f-ij-constraints} explicitly and directly targets the leading
eigenfunctions rather than basis of eigenspaces. Also, as opposed to the
approach in~\cite{state-free-vampnets}, training with either loss \eqref{loss-for-generator} or \eqref{loss-for-transfer-operator} does not require backpropagation on matrix eigenvalue problems.

\section{Encoder as CVs for low-dimensional representation of molecular configurations}
\label{sec-ae}

In this section, we briefly discuss autoencoders in the context of CV identification for molecular dynamics.

An autoencoder~\cite{ae-kramer} on $\mathbb{R}^d$ is a function $f$ that maps
an input data $x \in \mathbb{R}^d$ to an output $y\in \mathbb{R}^d$ by
passing through an intermediate (latent) space $\mathbb{R}^k$, where $1 \le k < d$.  
It can be written in the form $f=f_{dec}\circ f_{enc}$, where $f_{enc}: \mathbb{R}^d\rightarrow \mathbb{R}^k$ and $f_{dec}: \mathbb{R}^k\rightarrow \mathbb{R}^d$ are called an encoder and a decoder, respectively. 
The integer $k$ is called the encoded dimension (resp. bottleneck dimension). In other words, under the mapping of the autoencoder $f$, the input $x$ is first
mapped to a state $z$ in the latent space $\mathbb{R}^k$ by the encoder $f_{enc}$, which is then mapped to $y$ in the original space by the decoder $f_{dec}$.
In practice, both the encoder and the decoder are represented by artificial neural networks (see Figure~\ref{fig-ae}).
Given a set of data $x^{(0)}, x^{(1)}, \dots, x^{(N-1)} \in \mathbb{R}^d$, they are typically trained by minimising the empirical reconstruction error
\begin{equation}
  \mathrm{Loss}^{AE}(f_{enc}, f_{dec}) = \frac{1}{N}\sum_{i=0}^{N-1} |f_{dec}\circ f_{enc}(x^{(i)}) - x^{(i)}|^2\,.
  \label{ae-reconstruction-error}
\end{equation}
  \begin{figure}[h]
     \centering
\includegraphics[width=0.50\linewidth]{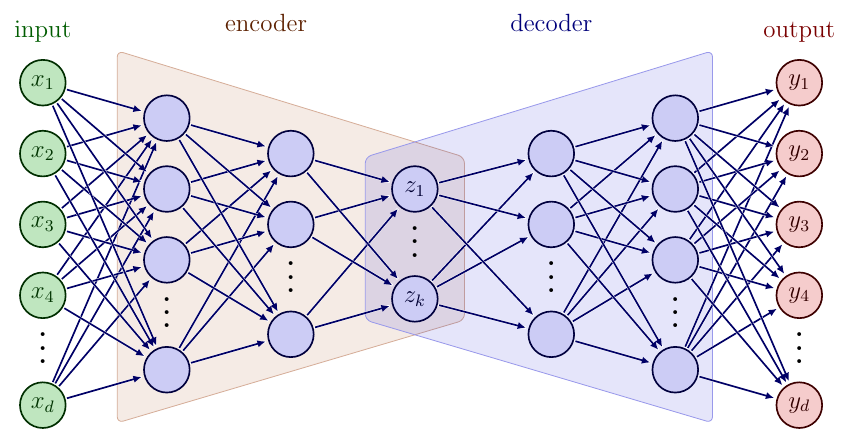}
    \caption{Illustration of autoencoder represented by an artificial neural
    network. An input $x=(x_1,x_2,\dots, x_d)^\top \in \mathbb{R}^d$ is
    first mapped to $z=(z_1,\dots, z_k)^\top\in \mathbb{R}^k$ in the latent
    space, which is then mapped to the output $y=(y_1,y_2,\dots, y_d)^\top \in \mathbb{R}^d$. \label{fig-ae}}
  \end{figure}

In the context of CV identification for molecular systems, the trained encoder
$f_{enc}$ is used to define the CV map, i.e.\ $\xi=f_{enc}$. Note that the
loss \eqref{ae-reconstruction-error} is invariant under reordering of training data.
For trajectory data, instead of \eqref{ae-reconstruction-error} it would be beneficial to employ a loss that incorporates temporal information in the data.
In this regard, several variants, such as time-lagged autoencoders~\cite{time-lagged-ae-noe,time-lagged-ae-ferguson} and the extended autoencoder using
committor function~\cite{extended-ae-bolhuis}, have been proposed in order to
learn low-dimensional representations of the system that can capture its essential dynamics.

\paragraph{Connection with PCA}
An autoencoder can be viewed as a nonlinear generalisation of PCA, which is a widely used technique for dimensionality reduction. 
To elucidate their connection, let us assume without loss of generality that the data satisfies $\frac{1}{N} \sum_{i=0}^{N-1} x^{(i)}=0$ and recall that the PCA algorithm actually solves the optimisation problem 
\begin{equation}
  \min_{U_k} \sum_{i=0}^{N-1} | x^{(i)} - U_kU_k^\top x^{(i)}|^2
  \label{pca-reconstruction-error}
\end{equation}
among matrices $U_k \in \mathbb{R}^{d\times k}$ with $k$ orthogonal unit vectors as columns~\cite[Section 14.5]{hastie2009elements}.
Comparing \eqref{ae-reconstruction-error} and \eqref{pca-reconstruction-error}, it is apparent that autoencoder can be considered as a nonlinear generalisation of PCA and it reduces to PCA when the encoder and decoder are restricted to linear maps given by $f_{enc}(x) := U_k^\top x$ and $f_{dec}(z) := U_k z$ for $x\in \mathbb{R}^d, z\in\mathbb{R}^k$, respectively.

\paragraph{Characterisation of time-lagged autoencoders}
We give a characterisation of the optimal encoder and the optimal decoder in the time-lagged autoencoders~\cite{time-lagged-ae-noe}. 
Assume that the data $x^{(0)}, x^{(1)}, \dots, x^{(N-1)}$ comes from the
trajectory of an underlying ergodic process with invariant measure $\mu$ in \eqref{mu-invariant} sampled at time $i\Delta t$, where $\Delta t>0$ and
$i=0,1,\dots, N-1$. Also assume that, for some $\tau>0$, the state $y$ of the
underlying system at time $\tau$ given its current state $x$ can be described as an ergodic Markov jump process with transition density
$p_{\tau}(y|x)$ (see the discussion on transfer operator in Section~\ref{subsec-operators}).
For simplicity, we assume $\tau=j\Delta t$ for some integer $j>0$. The time-lagged autoencoder is an autoencoder trained with the loss 
\begin{equation}
  \mathrm{Loss}^{AE}_\tau(f_{enc}, f_{dec}) = \frac{1}{N-j}\sum_{i=0}^{N-j-1} |f_{dec}\circ f_{enc}(x^{(i)}) - x^{(i+j)}|^2 \,,
  \label{ae-reconstruction-error-tau}
\end{equation}
which reduces to the standard reconstruction loss \eqref{ae-reconstruction-error} when $j=0$. 

Let us consider the limit of \eqref{ae-reconstruction-error-tau} when
$N\rightarrow +\infty$.  Given the encoder $f_{enc}$ and $z \in \mathbb{R}^k$, denote by $\mu^{f_{enc}}_z$ the conditional measure
on the level set $\Sigma^{f_{enc}}_z:=\{x\in \mathbb{R}^d|f_{enc}(x)=z\}$ (also see \eqref{mu-z}):
\begin{align}
    d\mu^{f_{enc}}_z(x)  =  \frac{1}{Q^{f_{enc}}(z)} \frac{e^{-\beta V(x)}}{Z} \delta(f_{enc}(x)-z) dx\,,
  \label{mu-z-ae}
  \end{align}
 where $Q^{f_{enc}}(z)$ is a normalising constant and satisfies $\int_{z\in\mathbb{R}^k} Q^{f_{enc}}(z) dz=1$.
Using \eqref{mu-z-ae} and ergodicity, we have
\begin{equation}
  \begin{aligned}
    \mathrm{Loss}^{AE}_\tau(f_{enc}, f_{dec}) =& \lim\limits_{N\rightarrow +\infty} \frac{1}{N-j}\sum_{i=0}^{N-j-1} |f_{dec}\circ f_{enc}(x^{(i)}) - x^{(i+j)}|^2  \\
    =& \int_{x\in \mathbb{R}^d}\int_{y\in \mathbb{R}^d} |f_{dec}\circ f_{enc}(x) - y|^2 p_\tau(y|x) d\mu(x) dy\\
    =& \int_{x\in \mathbb{R}^d}\int_{y\in \mathbb{R}^d}\Big[\int_{z\in \mathbb{R}^k} |f_{dec}(z)- y|^2\delta(f_{enc}(x)-z)dz\Big] p_\tau(y|x)  d\mu(x) dy \\
    =& \int_{z\in \mathbb{R}^k}\Big[\int_{y\in \mathbb{R}^d}\int_{x\in \Sigma^{f_{enc}}_z} |f_{dec}(z)- y|^2 p_\tau(y|x) d\mu^{f_{enc}}_z(x) dy\Big] Q^{f_{enc}}(z) dz\\
    =& \int_{z\in \mathbb{R}^k} \left[\mathbf{E}_{y \sim
    \mu^{f_{enc}}_{z,\tau}}|f_{dec}(z)- y|^2\right]  Q^{f_{enc}}(z) dz \\
    =& \mathbf{E}_{z \sim \widetilde{\mu}^{f_{enc}}} \Big[\mathbf{E}_{y \sim \mu^{f_{enc}}_{z,\tau}}|f_{dec}(z)- y|^2\Big] \,,
  \end{aligned}
  \label{ae-reconstruction-error-tau-1}
\end{equation}
where $d\widetilde{\mu}^{f_{enc}}= Q^{f_{enc}}(z) dz$ is a probability measure
on $\mathbb{R}^k$, and we have denoted by $\mu^{f_{enc}}_{z,\tau}$ the probability measure on $\mathbb{R}^d$ defined by 
\begin{equation}
  d\mu^{f_{enc}}_{z,\tau}(y)= \Big(\int_{x\in \Sigma^{f_{enc}}_z} p_\tau(y|x) d\mu^{f_{enc}}_z(x)\Big) dy\,, \quad y \in \mathbb{R}^d\,.
  \label{mu-z-tau}
\end{equation}
Using the simple identity
\begin{equation*}
  \min_{y' \in \mathbb{R}^d} \mathbf{E}_{y \sim \mu^{f_{enc}}_{z,\tau}}|y- y'|^2 = \mathbf{Var}_{y\sim \mu^{f_{enc}}_{z,\tau}} (y)
\end{equation*}
where the right hand side is the variance of $y$ distributed according to $\mu^{f_{enc}}_{z,\tau}$ and the minimum is attained at $y'= \mathbf{E}_{y \sim
\mu^{f_{enc}}_{z,\tau}} (y)$, we can finally write the minimisation of \eqref{ae-reconstruction-error-tau-1} as 
\begin{equation}
  \begin{aligned}
     \min_{f_{enc},f_{dec}} \mathrm{Loss}^{AE}_\tau(f_{enc}, f_{dec}) =& \min_{f_{enc}} \min_{f_{dec}} \mathbf{E}_{z \sim
    \widetilde{\mu}^{f_{enc}}} \Big[\mathbf{E}_{y \sim \mu^{f_{enc}}_{z,\tau}}|f_{dec}(z)- y|^2\Big] \\
    =& \min_{f_{enc}} \mathbf{E}_{z \sim \widetilde{\mu}^{f_{enc}}} \Big[\min_{y'=f_{dec}(z)} \mathbf{E}_{y \sim \mu^{f_{enc}}_{z,\tau}}|y'- y|^2\Big] \\
    =& \min_{f_{enc}} \mathbf{E}_{z \sim \widetilde{\mu}^{f_{enc}}}\Big[\mathbf{Var}_{y\sim \mu^{f_{enc}}_{z,\tau}} (y)\Big] \,.
  \end{aligned}
  \label{ae-reconstruction-error-tau-2}
\end{equation}

Note that \eqref{mu-z-tau} is the distribution of $y$ at time $\tau$ starting from points $x$ on the levelset $\Sigma^{f_{enc}}_z$ distributed
according to the conditional measure $\mu^{f_{enc}}_z$. To summarize, \eqref{ae-reconstruction-error-tau-2} implies that, when $N\rightarrow +\infty$,
training time-lagged autoencoder yields (in theory) the encoder map $f_{enc}$
that minimises the average variance of the future states $y$ (at time $\tau$) of points $x$ on $\Sigma^{f_{enc}}_z$ distributed according to
$\mu^{f_{enc}}_z$, and the decoder that is given by the mean of the future
states $y$, i.e.\ $f_{dec}(z) = \mathbf{E}_{y\sim\mu^{f_{enc}}_{z,\tau}} (y)$ for $z\in \mathbb{R}^k$.
Similar results hold for the standard autoencoder with the reconstruction loss \eqref{ae-reconstruction-error}.
In fact, choosing $\tau=0$ in the above derivation leads to the conclusion that the optimal encoder $f_{enc}$ minimises the average variance of the measures $\mu^{f_{enc}}_z$ on the levelsets.

To conclude, we note that although the loss \eqref{ae-reconstruction-error-tau} in time-lagged autoencoders encodes temporal information of data, 
from the characterisation \eqref{ae-reconstruction-error-tau-2} it is not clear whether this temporal information is sufficient in order to yield encoders that are suitable to define good CVs (in the sense discussed in Section~\ref{sec-eigen}). Our characterisation of time-lagged autoencoders is in line with the previous study on the time-lagged autoencoders~\cite{time-lagged-ae-ferguson}, where 
the authors analysed the capability and limitations of the time-lagged autoencoders in finding the slowest mode of the system, and proposed modifications of time-lagged autoencoders (in order to discover the slowest mode). In the next section, we will further compare autoencoders and eigenfunctions on concrete numerical examples. 

\section{Numerical Examples}
\label{sec-examples}

In this section, we show numerical results of eigenfunctions and autoencoders on two simple two-dimensional systems.
For eigenfunctions, we only consider the transfer operator and the loss~\eqref{loss-for-transfer-operator} due to its simplicity. 
Numerical study on computing eigenfunctions for the generator using the loss
\eqref{loss-for-generator} can be found in \cite{eigenpde-by-nn-ZHANG}. The code for training neural networks is implemented in PyTorch.

\subsection{First example} 
\label{subsec-ex1}

The first system satisfies the SDE \eqref{sde-x} with $\beta=4.0$ and the potential (taken from~\cite{effective_dynamics})
\begin{equation}
  V(x_1, x_2) = (x_1^2-1)^2 + \frac{1}{\epsilon} (x_1^2 + x_2 - 1)^2,  \quad (x_1, x_2)^\top \in \mathbb{R}^2\,,
\end{equation}
where we choose $\epsilon=0.5$. As shown in Figure~\ref{ex1-system-data},
there are two metastable regions in the state space, and the system can
transit from one to the other through a curved transition channel. We sampled
the trajectory of \eqref{sde-x} for $10^5$ steps using Euler-Maruyama scheme with time step-size $\Delta t=0.005$.
The sampled states were recorded every $2$ steps. This resulted in a dataset
consisting of $5\times 10^4$ states, which were used in training neural
networks~\footnote{Note that the empirical distribution of the data (shown in Figure~\ref{ex1-system-data}) slightly differs from the true invariant distribution $\mu$ of the dynamics.
However, there are sufficiently many samples in both metastable regions and also in the transition region.
In particular, the discrepancy between the empirical distribution and the true
invariant distribution is not the main factor that determines the quality of the numerical results.}.

\begin{figure}[h!]
  \begin{subfigure}[c]{0.28\textwidth}
\includegraphics[width=1.0\textwidth]{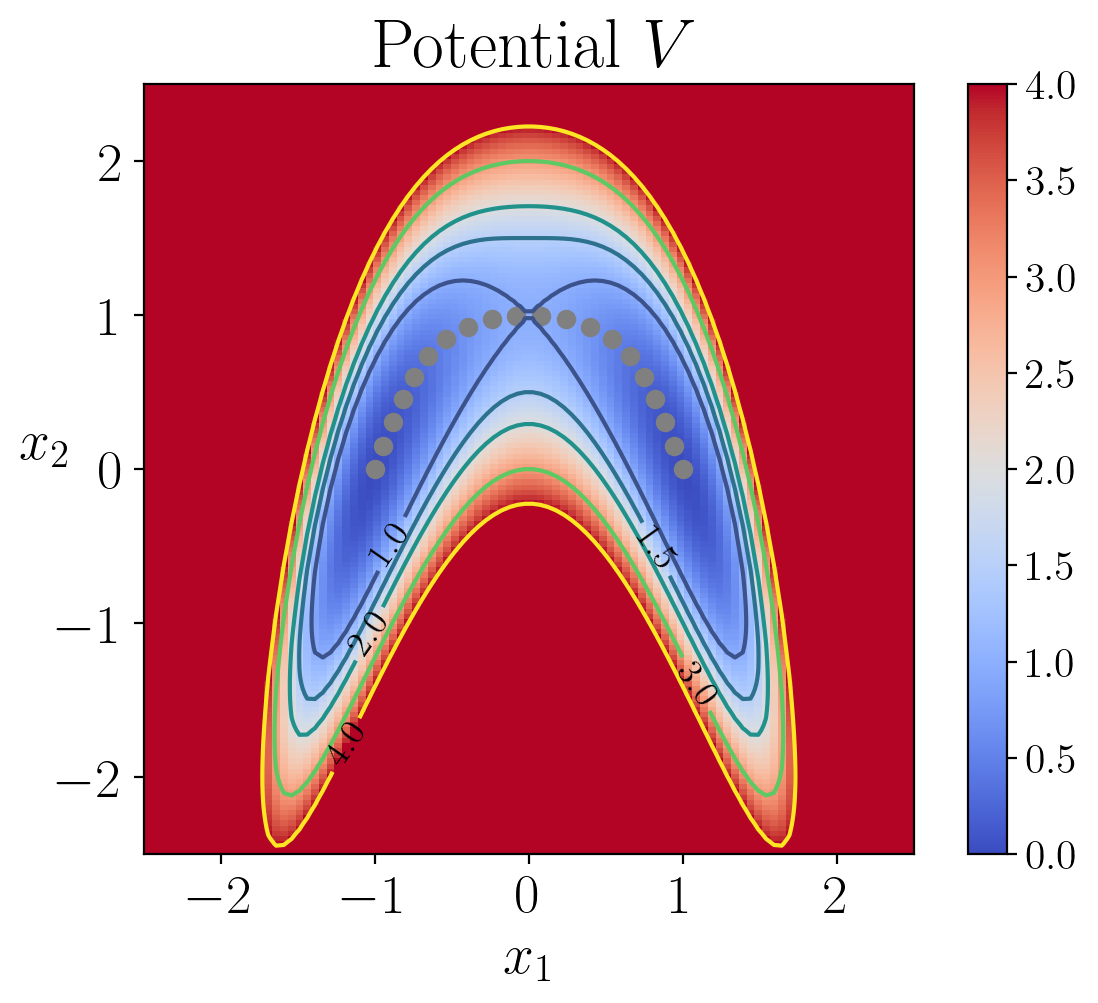}
  \end{subfigure}
  \begin{subfigure}[c]{0.28\textwidth}
\includegraphics[width=1.0\textwidth]{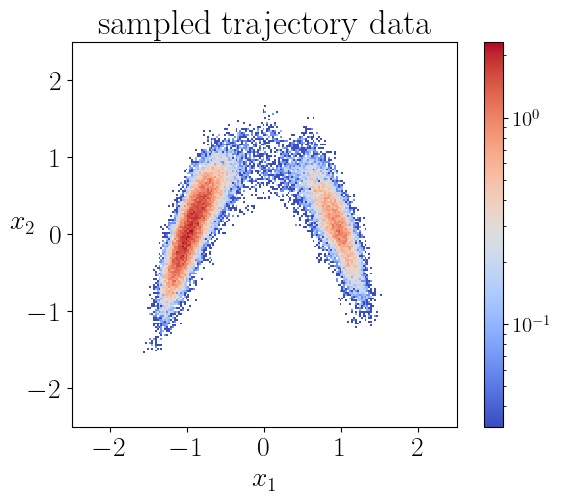}
  \end{subfigure}
  \begin{subfigure}[c]{0.28\textwidth}
\includegraphics[width=1.0\textwidth]{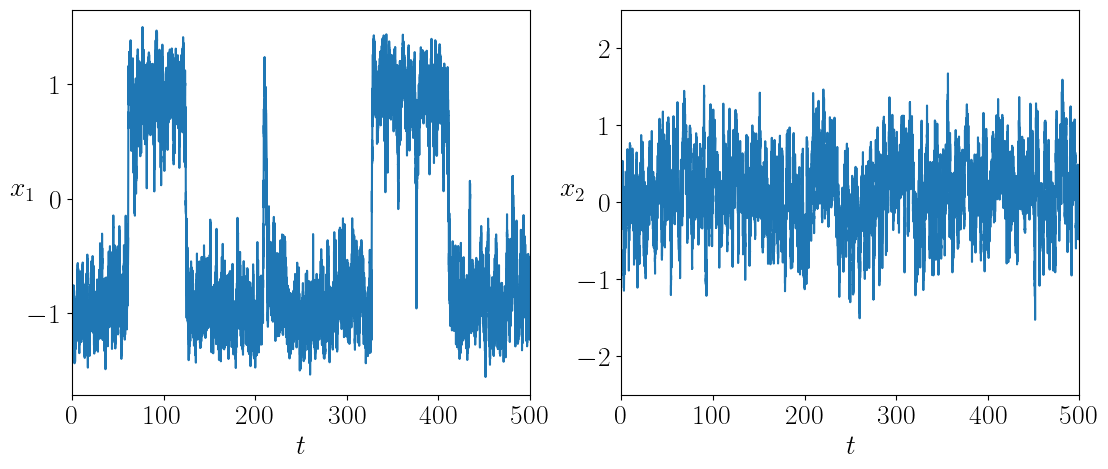}
  \end{subfigure}
\centering
  \caption{First example. Left: potential $V$ of the system and the transition
  path. Middle: histogram of the sampled data. Right: coordinates of sampled data as a function of time. \label{ex1-system-data}}
\end{figure}

We trained neural networks with the loss \eqref{ae-reconstruction-error} for standard autoencoders and the loss \eqref{ae-reconstruction-error-tau} for time-lagged autoencoders. In each test, since the total dimension is $2$, we chose the bottleneck
 dimension $k=1$. The encoder is represented by a neural network that has an
 input layer of size $2$, an output layer of size $1$, and $4$ hidden layers
 of size $30$ each. The decoder is represented by a neural network that has an
 input layer of size $1$, an output layer of size $2$, and $3$ hidden layers
 of size $30$ each. We took \textrm{tanh} as activation function in all neural
 networks. For the training, we used Adam optimiser~\cite{adam-KingmaB14} with batch size $2\times 10^4$ and learning rate $0.005$. 
 The random seed was fixed to be $2046$ and the total number of training epochs
 was set to $500$. Fig.~\ref{ex1-ae} shows the trained autoencoders with different lag-times.
As one can see there, for both the standard autoencoder ($\tau=0.0$) and the time-lagged autoencoder with a small lag-time ($\tau=0.5$), the contour lines of the
trained encoder match well with the stiff direction of the potential. The curves determined by the image of the decoders are also close to the
transition path. However, the results for time-lagged autoencoders become unsatisfactory when the lag-time was chosen as $1.0$ and $2.0$.

We also learned the first eigenfunction $\varphi_1$ of the transfer operator
using the loss \eqref{loss-for-transfer-operator}, where we chose $k=1$, the coefficient $\omega_1=1.0$, lag-time $\tau=1.0$, and the penalty constant $\alpha=10.0$. 
The same dataset and the same training parameters as in the training of
autoencoders were used, except that for the eigenfunction we employed a neural network that has $3$ hidden layers of size $20$ each.
The learned eigenfunction is shown in Figure~\ref{ex1-eigenfunc}. We can see
that the eigenfunction is indeed capable of identifying the two metastable regions and its contour lines are well aligned with the stiff directions of the potential in the transition region (but not inside the metastable regions).

\begin{figure}[h!]
  \begin{subfigure}{0.24\textwidth}
\includegraphics[width=1.0\textwidth]{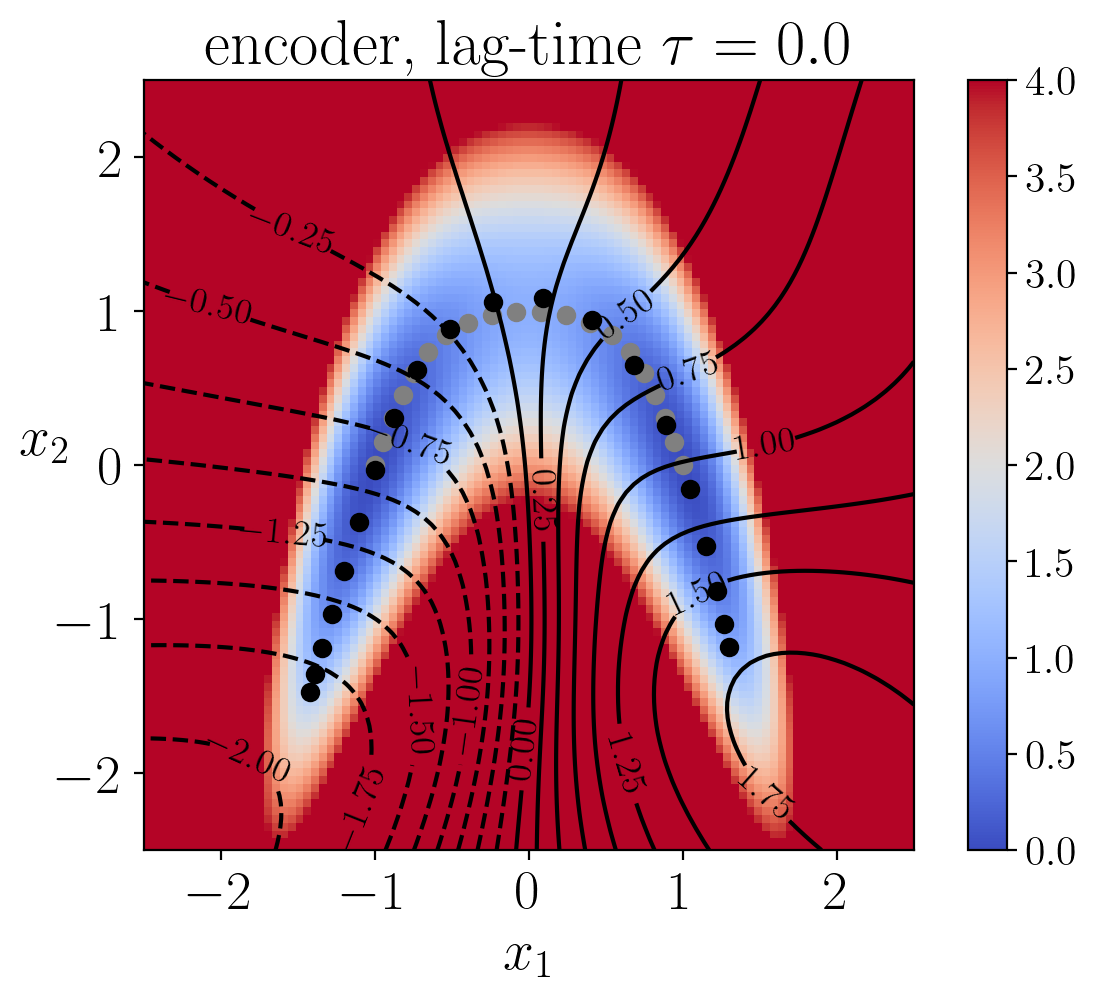}
    \caption{$\tau=0.0$}
  \end{subfigure}
  \begin{subfigure}{0.24\textwidth}
\includegraphics[width=1.0\textwidth]{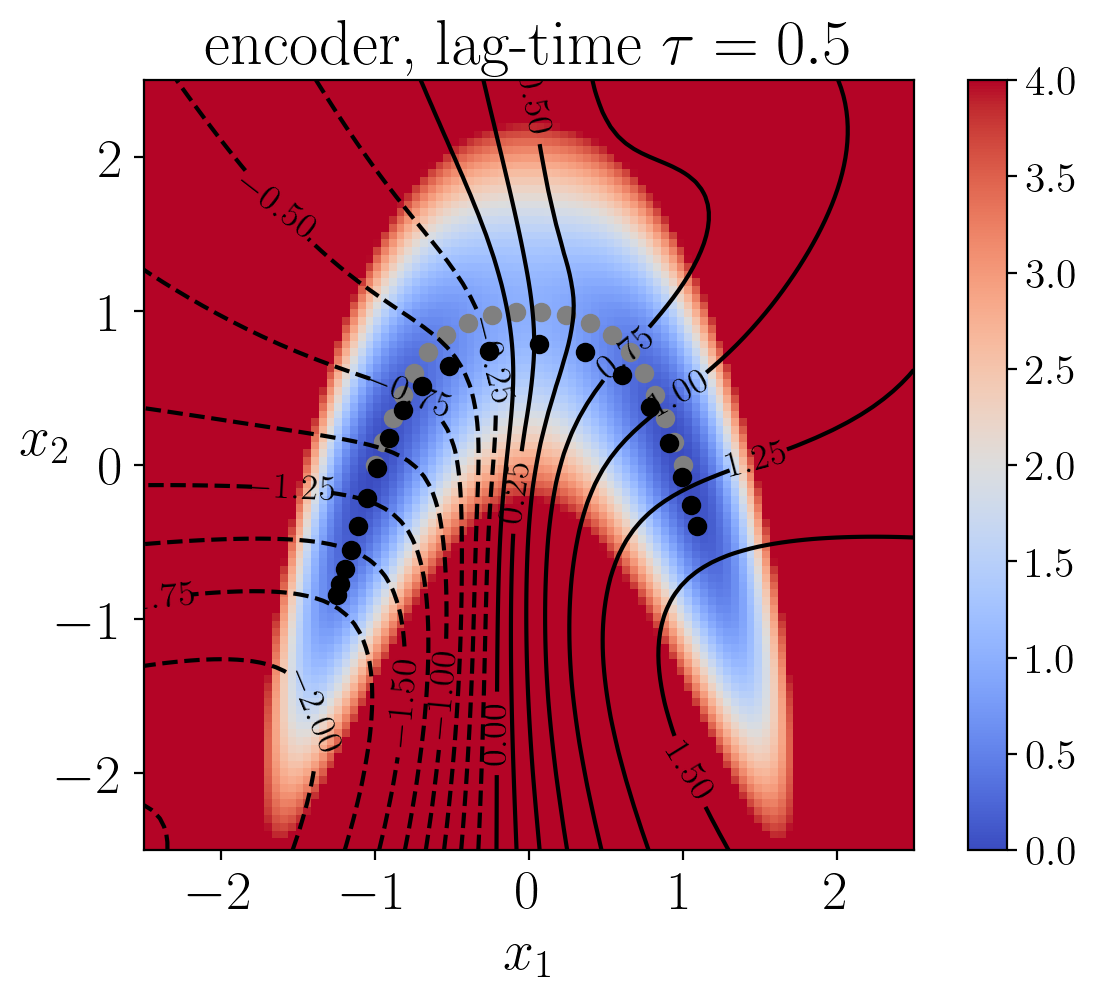}
    \caption{$\tau=0.5$}
  \end{subfigure}
  \begin{subfigure}{0.24\textwidth}
\includegraphics[width=1.0\textwidth]{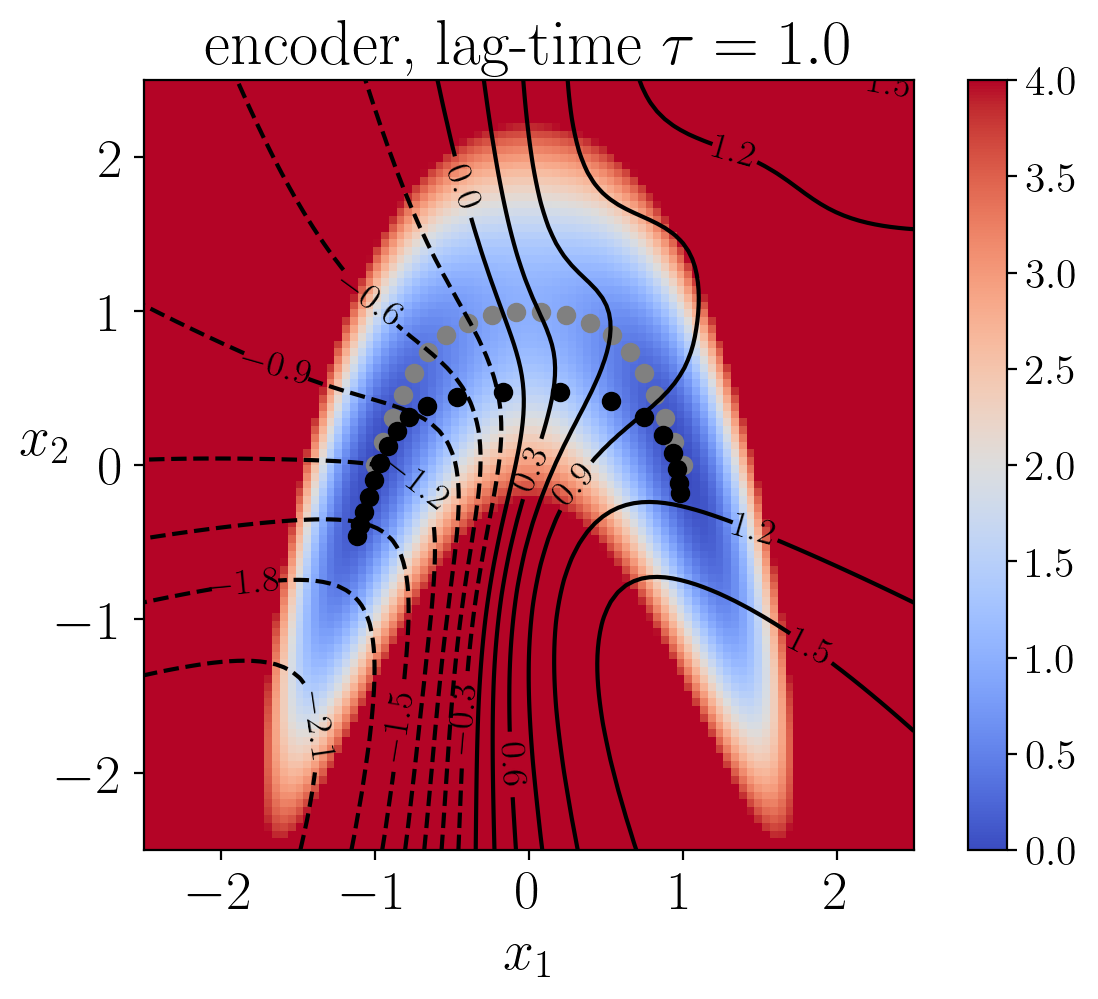}
    \caption{$\tau=1.0$}
  \end{subfigure}
  \begin{subfigure}{0.24\textwidth}
\includegraphics[width=1.0\textwidth]{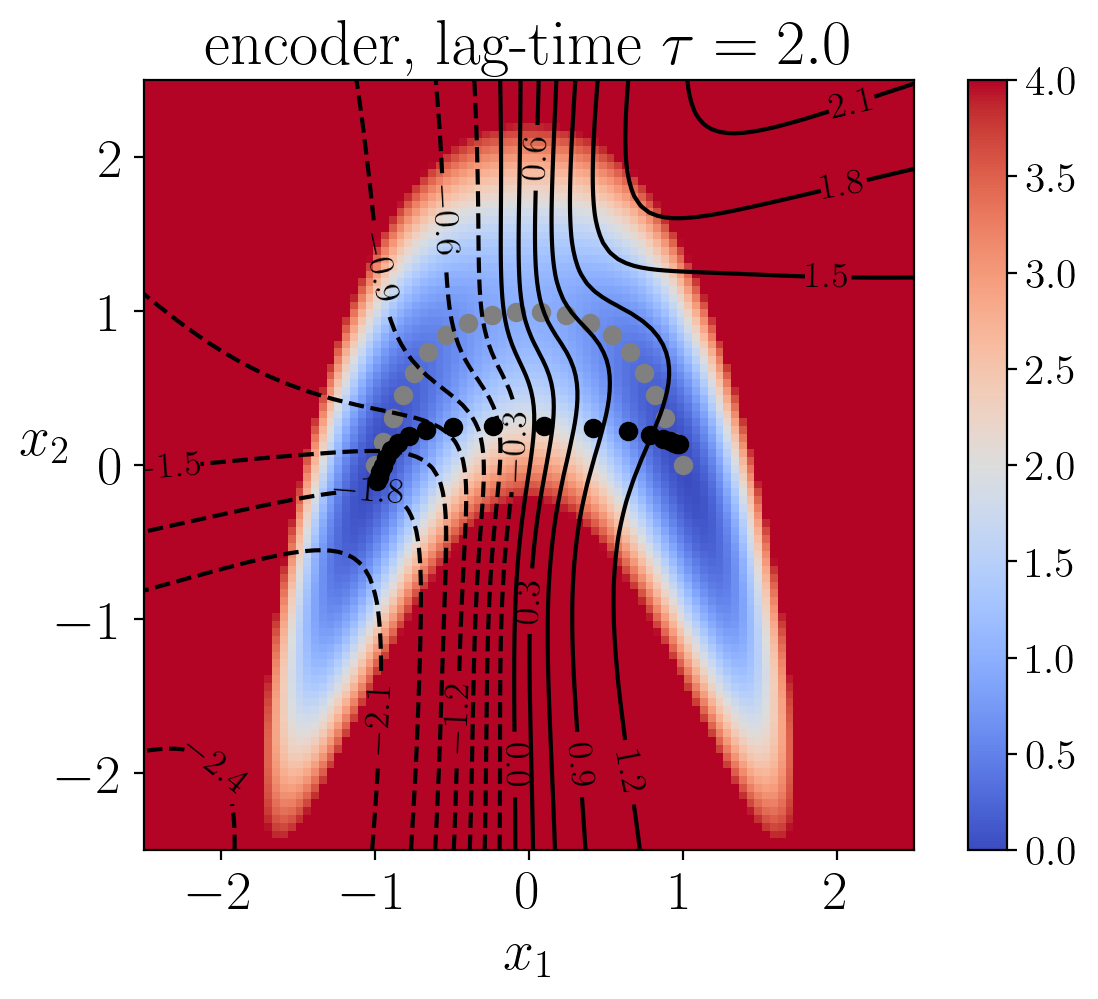}
    \caption{$\tau=2.0$}
  \end{subfigure}
\centering
  \caption{First example. (a) Contour lines of encoder trained with the
  standard reconstruction loss \eqref{ae-reconstruction-error}. (b), (c) and
  (d): Contour lines of encoders trained with the loss \eqref{ae-reconstruction-error-tau} and lag-times $\tau=0.5, 1.0, 2.0$,
  respectively. In each plot, the curve shown in gray dots is the minimal energy path computed using string method~\cite{weinan02}, whereas the curve shown in black dots is the curve given by the image of the trained decoder. \label{ex1-ae}}
\end{figure}

\begin{figure}[h!]
\includegraphics[width=0.28\textwidth]{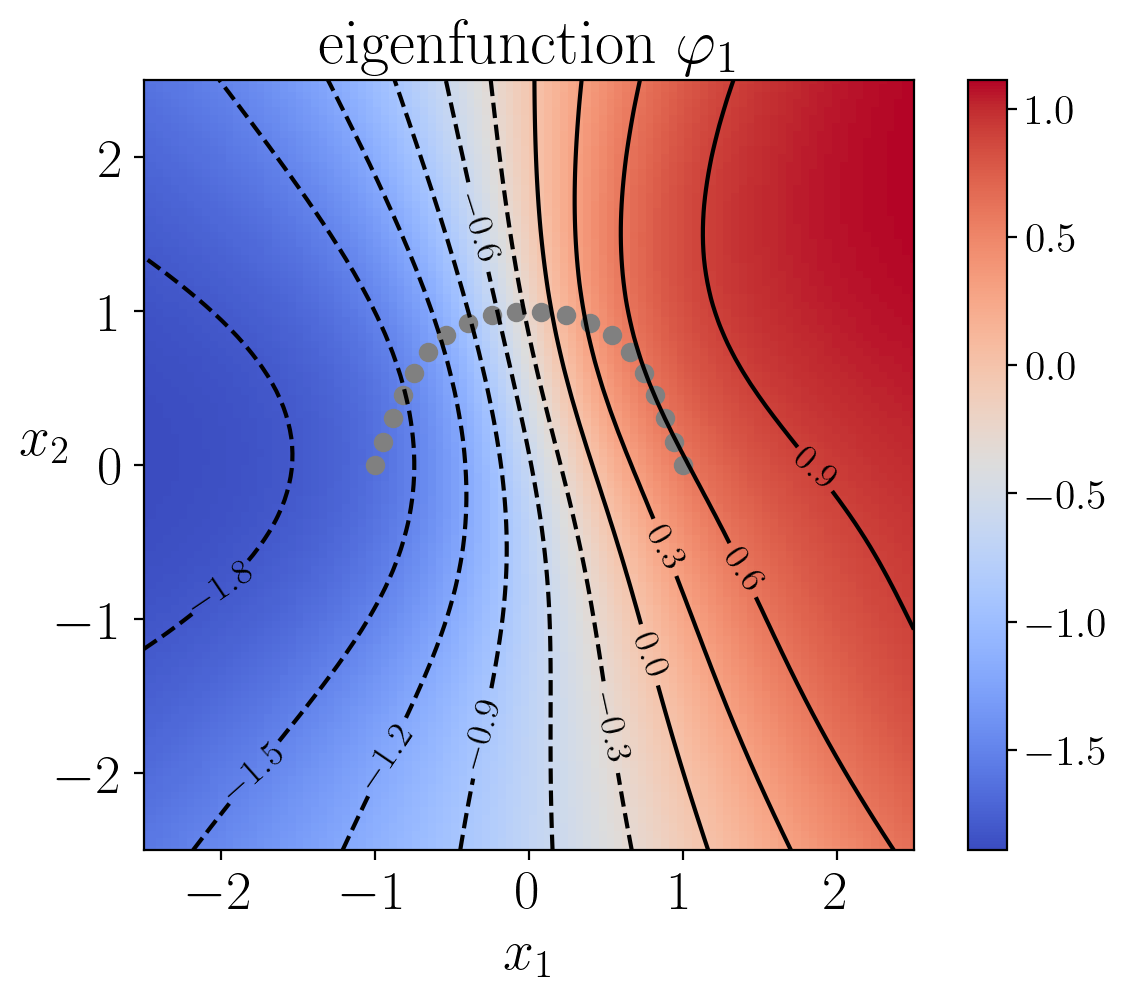}
\centering
  \caption{First example. Eigenfunction $\varphi_1$ of the transfer operator with $\tau=1.0$ trained using the loss
  \eqref{loss-for-transfer-operator}.\label{ex1-eigenfunc}}
\end{figure}

\subsection{Second example} 
\label{subsec-ex2}

In the second example, we consider a system that satisfies the SDE
\eqref{sde-x} with $\beta=1.5$ and the potential 
\begin{equation*}
  V(x_1, x_2) = \frac{e^{1.5 x_2^2}}{1 + e^{5(x_1^2-1)}} - 4 e^{-4 (x_1-2)^2-0.4x_2^2} -5 e^{-4 (x_1+2)^2-0.4x_2^2} + 0.2 (x_1^4 + x_2^4) + 0.5 e^{-2x_1^2} \,,
\end{equation*}
for $(x_1,x_2)^\top\in \mathbb{R}^2$. As shown in
Figure~\ref{ex2-system-data}, there are again two metastable regions. The
region on the left contains the global minimum point of $V$, and the region on the right contains a local minimum point of $V$.

To prepare training data, we sampled the trajectory of \eqref{sde-x} using Euler-Maruyama scheme with the
same parameters as in the previous example, except that in this example we
sampled in total $5 \times 10^5$ steps. By recording states every $2$ steps, we obtained a dataset of size $2.5\times 10^5$.

We learned the autoencoder with the standard reconstruction loss \eqref{ae-reconstruction-error} and the eigenfunction $\varphi_1$ of transfer operator with loss \eqref{loss-for-transfer-operator}, respectively.
For both autoencoder and eigenfunction, we used the same network architectures as in the previous example. We also used the same training parameters, except that in this example a larger batch-size $10^5$ was used and the total number of training epochs was set to $1000$. The lag-time for transfer operator is $\tau=0.5$.  Figure~\ref{ex2-ae-eigenfunc} shows the learned autoencoder and the eigenfunction $\varphi_1$.
As one can see there, since the autoencoder is trained to minimise the reconstruction error and most sampled data falls into the two metastable
regions, the contour lines of the learned encoder match the stiff directions of the potential in the metastable regions, but the transition region is poorly characterised.
On the contrary, the learned eigenfunction $\varphi_1$, while being close to constant inside the two metastable regions, gives a good parameterisation of the transition region. 
We also tried time-lagged autoencoders with lag-time $\tau=0.5$ and $\tau=1.0$
(results are not shown here). But, we were not successful in obtaining satisfactory results as compared to the learned eigenfunction in Figure~\ref{ex2-ae-eigenfunc}.

\begin{figure}[h!]
  \begin{subfigure}[c]{0.28\textwidth}
\includegraphics[width=1.0\textwidth]{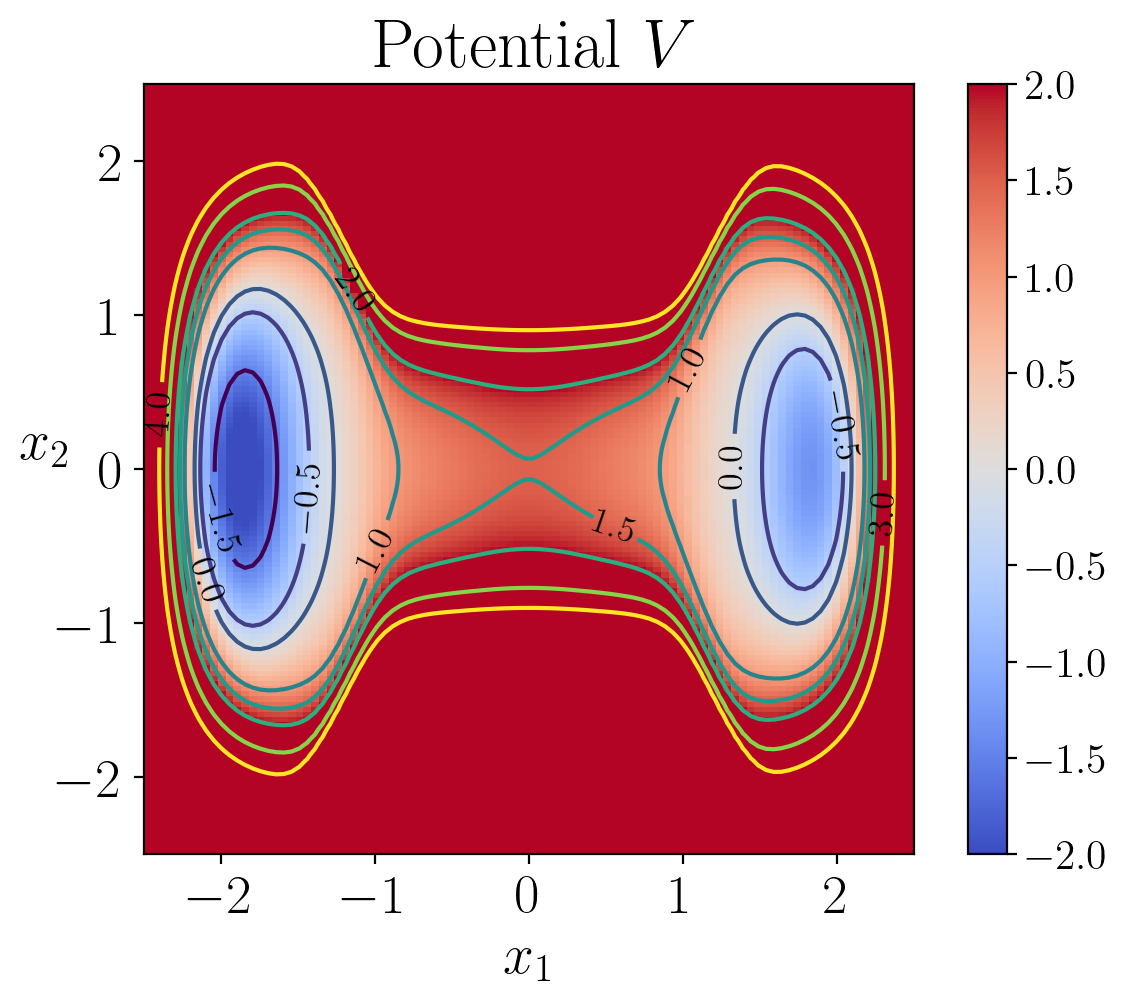}
  \end{subfigure}
  \begin{subfigure}[c]{0.28\textwidth}
\includegraphics[width=1.0\textwidth]{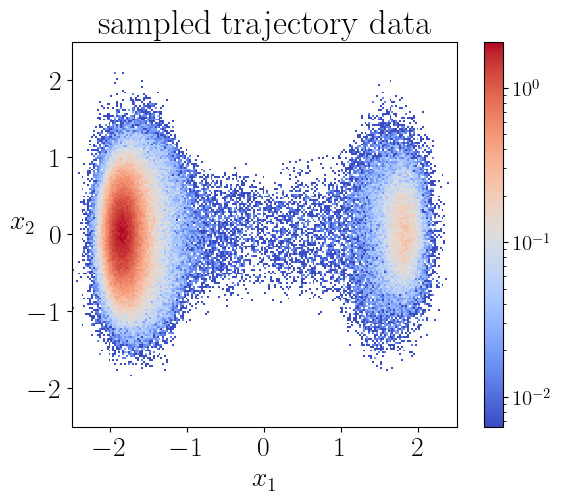}
  \end{subfigure}
  \begin{subfigure}[c]{0.28\textwidth}
\includegraphics[width=1.0\textwidth]{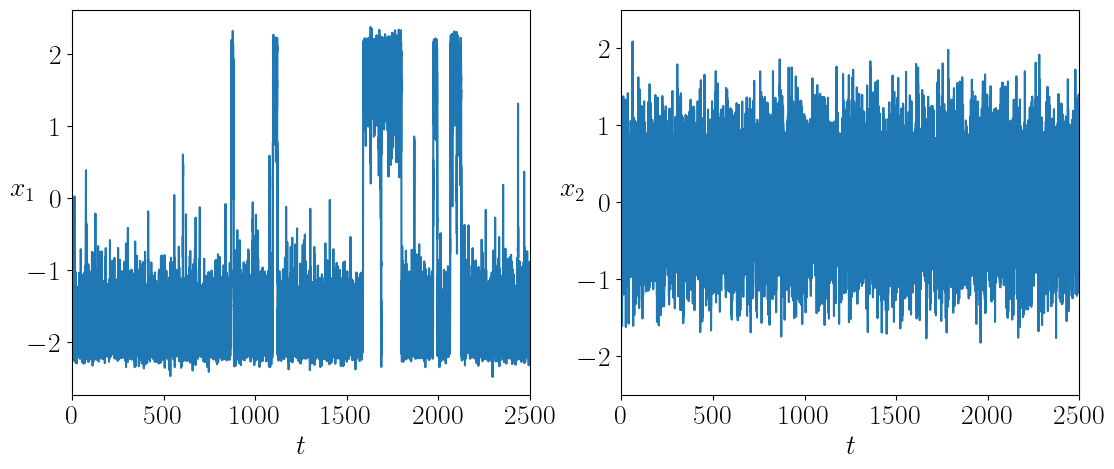}
  \end{subfigure}
\centering
  \caption{Second example. Left: potential $V$ of the system. Middle:
  histogram of the sampled data. Right: coordinates of the sampled data as a function of time (trajectory). \label{ex2-system-data}}
\end{figure}

\begin{figure}[h!]
  \begin{subfigure}{0.28\textwidth}
\includegraphics[width=1.0\textwidth]{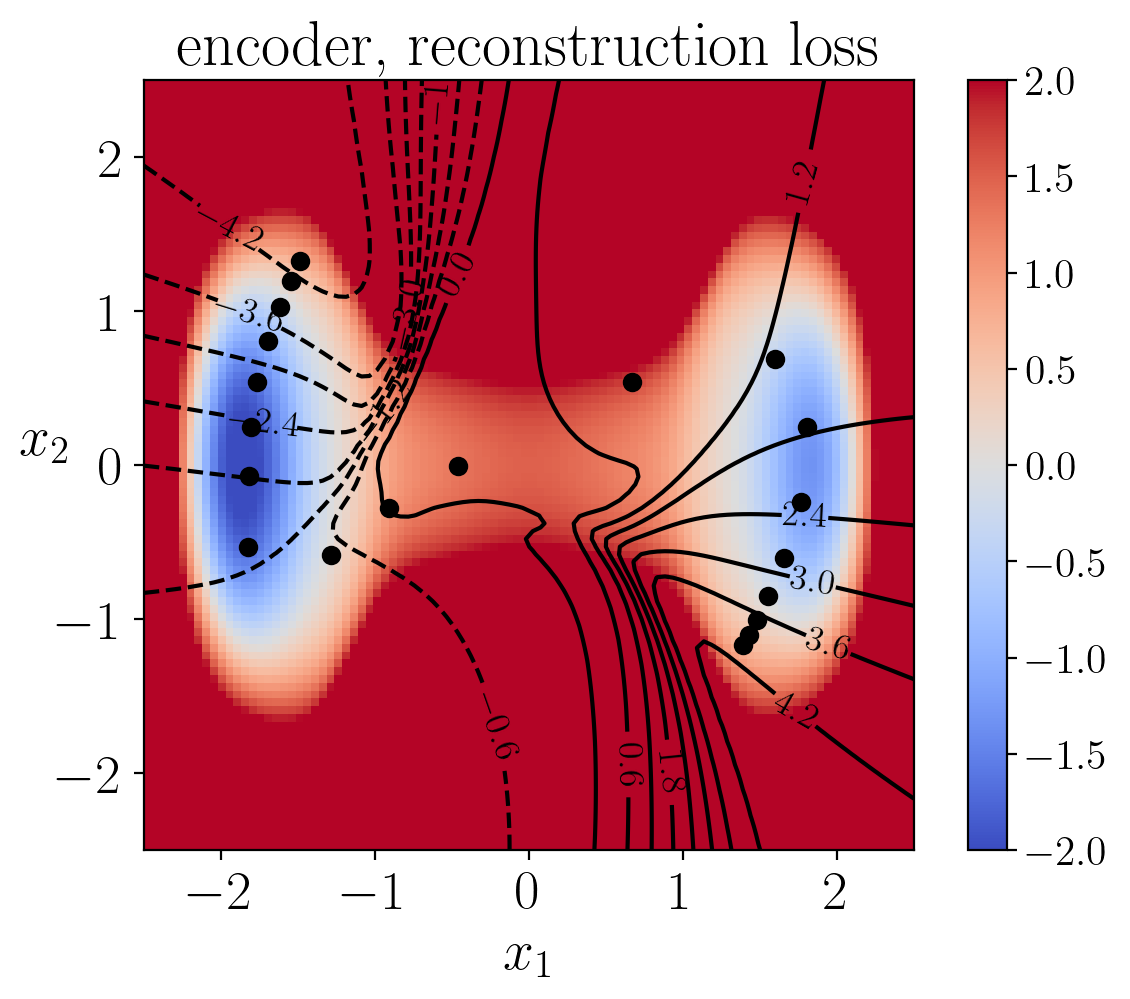}
  \end{subfigure}
  \begin{subfigure}{0.28\textwidth}
\includegraphics[width=1.0\textwidth]{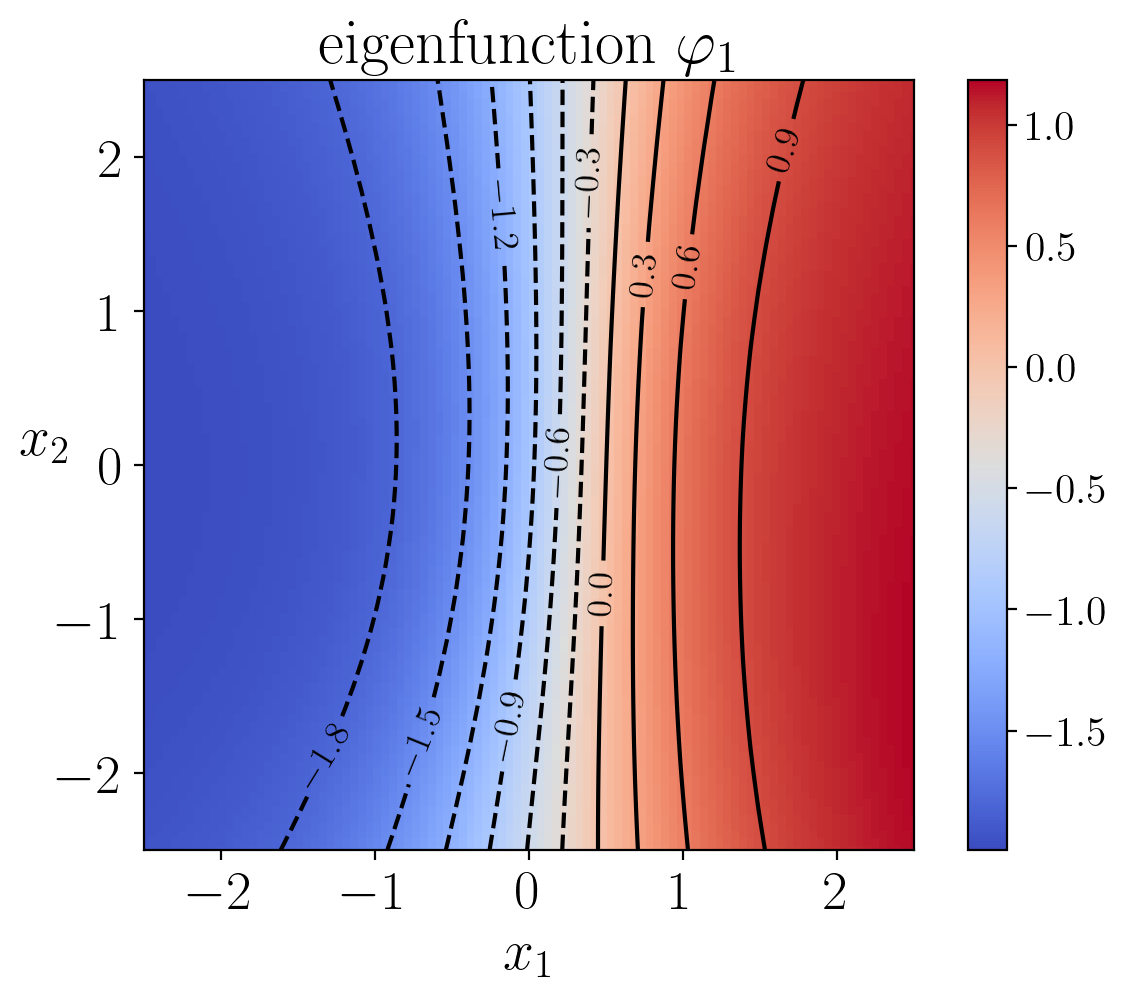}
  \end{subfigure}
\centering
  \caption{Second example. Left: Contour lines of the learned encoder map and the learned decoder curve (represented by black dots) are shown on top of the potential profile. Right: Eigenfunction
  $\varphi_1$ of the transfer operator with $\tau=0.5$ trained using the loss
  \eqref{loss-for-transfer-operator}.\label{ex2-ae-eigenfunc}}
\end{figure}

\begin{acknowledgement}
W.\ Zhang thanks Tony Leli\`evre and Gabriel Stolz for fruitful discussions on autoencoders. 
The work of C.\ Sch\"utte and W.\ Zhang is supported by the DFG under Germany's Excellence Strategy-MATH+: The Berlin Mathematics Research Centre (EXC-2046/1)-project ID:390685689.
\end{acknowledgement}

\appendix
\section{Proofs of Lemma~\ref{lemma-energy-transfer} and Lemma~\ref{lemma-nice-identity}}
\label{appsec-proofs}

\begin{proof}[Proof of Lemma~\ref{lemma-energy-transfer}]
  Applying the detailed balance condition and the second identity in \eqref{transfer-operator-and-semigroup}, we can derive 
  \begin{equation*}
    \begin{aligned}
      \mathcal{E}_\tau(f) =& \frac{1}{2} \int_{\mathbb{R}^d} \int_{\mathbb{R}^d} \big(f(y) - f(x))^2 p_\tau(y|x) \pi(x) dx dy \\
      =& \frac{1}{2}\int_{\mathbb{R}^d} \int_{\mathbb{R}^d} \big(f(y)^2 - 2f(x)f(y) + f(x)^2\big) p_\tau(y|x) \pi(x) dx dy \\
      =& \int_{\mathbb{R}^d} \int_{\mathbb{R}^d} f(x)^2 \pi(x) dx - \int_{\mathbb{R}^d} \int_{\mathbb{R}^d} f(x)f(y)  p_\tau(y|x) \pi(x) dx dy \\
      =& \int_{\mathbb{R}^d} \big[(I-\mathcal{T})f(x)\big] f(x) d\mu(x)\\
      =& \langle (I-\mathcal{T})f, f\rangle_\mu\,.
    \end{aligned}
    \end{equation*}
\end{proof}

\begin{proof}[Proof of Lemma~\ref{lemma-nice-identity}]
  It is straightforward to verify the identity (Bochner's formula)
  $\frac{1}{2} \Delta |\nabla f|^2 = \nabla(\Delta f) \cdot \nabla f + |\nabla^2 f|_F^2$, 
    where $\nabla^2 f$ denotes the matrix with entries $\frac{\partial^2
  f}{\partial x_i \partial x_j}$ for $1 \le i,j \le d$ and $|\nabla^2 f|_F$ is
  its Frobenius norm. Using this identity, together with \eqref{generator-l} and \eqref{l-selfadjoint}, we can derive 
    \begin{align*}
       \int_{\mathbb{R}^d} |\mathcal{L}f|^2 d\mu =&-\frac{1}{\beta} \int_{\mathbb{R}^d} \nabla f\cdot \nabla (\mathcal{L}f) d\mu\\
      =&-\frac{1}{\beta} \int_{\mathbb{R}^d} \nabla f\cdot \nabla (-\nabla V
      \cdot \nabla f + \frac{1}{\beta} \Delta f) d\mu \\
       =&\frac{1}{\beta} \int_{\mathbb{R}^d} \Big[\mbox{Hess}V(\nabla f,
       \nabla f) + \frac{1}{2}\nabla |\nabla f|^2\cdot \nabla V -
       \frac{1}{\beta} \nabla f \cdot \nabla \Delta f\Big] d\mu \\
       =&\frac{1}{\beta} \int_{\mathbb{R}^d} \Big[\mbox{Hess}V(\nabla f,
       \nabla f) + \frac{1}{2}\nabla |\nabla f|^2\cdot \nabla V -
       \frac{1}{\beta} \Big(\frac{1}{2} \Delta |\nabla f|^2 - |\nabla^2 f|^2_F \Big)\Big] d\mu \\
       =&\frac{1}{\beta} \int_{\mathbb{R}^d} \Big[\mbox{Hess}V(\nabla f,
       \nabla f) - \frac{1}{2}\mathcal{L} (|\nabla f|^2) +\frac{1}{\beta}
       |\nabla^2 f|^2_F \Big] d\mu \\
       =&\frac{1}{\beta} \int_{\mathbb{R}^d} \Big[\mbox{Hess}V(\nabla f,
       \nabla f) + \frac{1}{\beta}|\nabla^2 f|_F^2 \Big] d\mu \,,
    \end{align*}
  where the last equality follows from the fact that $\int \mathcal{L} |\nabla f|^2 d\mu = 0$.
\end{proof}

\bibliographystyle{siamplain}
\bibliography{reference}

\end{document}